%% file: ARXIV_VERSION.tex
\documentclass[
]{ceurart}

\usepackage{listings}

\usepackage{graphicx} 

\usepackage{amsthm}
\usepackage{float}
\usepackage{amssymb}
\usepackage{tabularx}
\usepackage{subcaption}

\newtheorem{theorem}{Theorem}[section]

\theoremstyle{definition}
\newtheorem{definition}{Definition}
\theoremstyle{observation}

\theoremstyle{example}

\begin{document}

\copyrightyear{2026}
\copyrightclause{Copyright for this paper by its authors. Use permitted under Creative Commons License Attribution 4.0 International (CC BY 4.0).}


\title{AI-Assisted Discovery of CNOT Decompositions for Quantum Permutations on Constrained Topologies}

\author[1,3]{Daniele {Lizzio Bosco}}[%
orcid=0009-0002-7372-6518,
email=lizziobosco.daniele@spes.uniud.it,
]

\author[1]{Jacopo Cossio}[%
orcid=0009-0006-1132-730X,
email=cossio.jacopo@spes.uniud.it,
]

\author[2]{Simone Faro}[%
orcid=0000-0001-5937-5796,
email=simone.faro@unict.it,
]

\author[2]{Francesco Pio Marino}[%
orcid=0000-0003-4722-9542,
email=francesco.marino@phd.unict.it,
]

\author[2,4]{Gabriele Messina}[%
orcid=0009-0004-9926-3691,
email=gabriele.messina@phd.unict.it,
]

\author[1]{Carla Piazza}[%
orcid=0000-0002-2072-1628,
email=carla.piazza@uniud.it,
]

\author[1]{Giuseppe Serra}[%
orcid=0000-0002-4269-4501,
email=giuseppe.serra@uniud.it,
]

\address[1]{Department of Mathematical, Computer and Physical Sciences, University of Udine, Via delle Scienze 206, 33100 Udine, Italy}

\address[2]{Department of Mathematics and Computer Science, University of Catania, Viale Andrea Doria 6, 95125 Catania, Italy}

\address[3]{Department of Biology, University of Naples Federico II, Corso Umberto I 40, 80138 Napoli, Italy}

\address[4]{Catania Astrophysical Observatory, Istituto Nazionale di Astrofisica,  Via Santa Sofia 78, 95123 Catania, Italy}


\fntext[1]{Supported by University of Catania, progetto PIACERI 2024-2026 - Linea di Intervento 1}

\begin{abstract}
Efficient quantum circuit compilation is a fundamental challenge for current quantum computers, where limited qubit connectivity can introduce substantial overhead in the implementation of logical operations. Recent AI-assisted synthesis methods have shown considerable potential for reducing circuit complexity, but their role in uncovering principles that can guide the design of new deterministic compilation algorithms remains largely unexplored. In this work, we investigate this perspective for the synthesis of quantum permutations into CNOT circuits under constrained qubit connectivity. Using AlphaCNOT as an exploratory tool, we identify a recurring decomposition for remote transpositions that improves on the standard construction based on adjacent SWAP routing. For two qubits at graph distance $d>1$, the discovered pattern yields a topology-legal implementation using at most $6d-4$ CNOT gates, compared with the $6d-3$ gates required by the corresponding SWAP-based construction. We formalize this observation through a constructive proof and use it to define Global Transpositions, a deterministic synthesis method that searches for a minimum-cost decomposition of an entire permutation into topology-aware logical transpositions. The resulting method subsumes optimal SWAP routing within its synthesis model and provides certified upper bounds on the topology-constrained CNOT cost. Our experiments also reveal optimized circuits that cannot be expressed as concatenations of independently synthesized transpositions. Instead, these circuits exploit intermediate linear transformations that temporarily leave the permutation subgroup and combine partial effects of multiple logical operations, a mechanism we call intermediate CNOT fusion. Overall, our results show that AI-assisted synthesis can support not only circuit optimization, but also the discovery of reusable identities, provable constructions, and previously hidden mechanisms in quantum circuit compilation.
\end{abstract}


\conference{AIQxQIA 2026: International Workshop on AI for Quantum and Quantum for AI | co-located with AIxIA 2026, Perugia, Italy}\maketitle


\input{sections/intro}
\input{sections/background}



\input{sections/discovery_local_primitives}

\input{sections/global_transposition_algorithm}

\input{sections/intermediate_fusion}

\input{sections/discussion}

\section*{Acknowledgements}
JC is supported by FSE/FVG PhD Grant on ``Computer Science and Artificial Intelligence'' (CUP G23C25000620008).
This work has been partially supported by INdAM-GNCS project \emph{Algebra lineare quantistica, state preparation e compilazione di circuiti quantistici} (CUP E53C25002010001) and by the regional project QUASAR-FVG \emph{Calcolo e simulazione quantistica: sviluppo, applicazioni e ricerca in Friuli Venezia Giulia} (CUP G23C25001510002).


\section*{Declaration on Generative AI}
During the preparation of this work, the authors used ChatGPT in order to perform grammar and spelling check, and to paraphrase and reword certain paragraphs. 
After using this tool, the authors reviewed and edited the content as needed and take full responsibility for the publication’s content.

\bibliography{biblio}

\end{document}

%% file: sections/intro.tex
\section{Introduction}

Limited qubit connectivity remains one of the main obstacles to the
efficient execution of quantum circuits on current hardware. A quantum
algorithm is typically designed in terms of logical qubits with no
restrictions on which pairs may interact, whereas a physical quantum
processor only permits two-qubit gates between adjacent vertices of a
coupling graph \cite{davis_towards_2020, kissinger_cnot_2020, wu_optimization_2023}. Consequently, the practical cost of executing a quantum
algorithm depends not only on its logical circuit, but also on how that
circuit is compiled onto the target topology.

A standard approach to connectivity-constrained compilation is to insert
SWAP gates whenever two logical qubits that must interact are not
physically adjacent. Since a SWAP is commonly decomposed into three CNOT
gates, minimizing the number of SWAPs provides a natural surrogate
objective for reducing the two-qubit gate count. This approach is
particularly relevant on current quantum devices, where two-qubit gates
typically dominate both the execution error and the compilation
overhead.

For permutation circuits, the relationship between SWAP and CNOT
synthesis is particularly direct. A permutation directly describes a
rearrangement of qubit states and can therefore be implemented as a
sequence of transpositions. In the absence of connectivity constraints,
this relationship is well understood. Bu \emph{et al.}~\cite{Bu_2025}
proved that a permutation of $n$ elements consisting of $k$ disjoint
cycles admits an optimal realization using exactly $3(n-k)$ CNOT gates.
Thus, under all-to-all connectivity, the optimal CNOT cost is completely
determined by the cycle structure of the target permutation.

Under constrained connectivity, however, SWAP minimization and CNOT
minimization are distinct problems. Chen \emph{et
al.}~\cite{Chen_2022} showed that a routing solution containing the
minimum possible number of SWAP gates does not necessarily yield a
minimum-CNOT implementation after each SWAP is expanded independently.
Although this establishes a separation between the two optimization
objectives, it leaves open a more structural question: \textit{which
opportunities available to direct CNOT synthesis are excluded when the
search is restricted to SWAP-based routing?}

We investigate this question using the constant-depth circuits for
quantum rotations and permutations introduced by Faro \emph{et
al.}~\cite{FaroPV24}. 
These constructions provide depth-optimal
implementations through structured layers of disjoint transpositions, and therefore constitute a strong
benchmark for determining whether circuits already optimized at the
permutation level can be further improved at the CNOT level.

To analyze these circuits, we employ AlphaCNOT~\cite{alphacnot}, an
AI-based optimizer that searches directly in the space of
topology-legal CNOT circuits. Unlike SWAP-based routing, AlphaCNOT is not
restricted to intermediate transformations that are themselves
permutations. Rather than treating the optimizer only as a black box, we
inspect its outputs to identify recurring structures that can be
formalized and incorporated into deterministic synthesis methods.

Our first observation concerns remote transpositions. A transposition
between two qubits at graph distance $d>1$ can be implemented through
$2d-1$ adjacent SWAPs, requiring $6d-3$ CNOT gates. AlphaCNOT repeatedly
produces implementations using $6d-4$ CNOTs. We formalize this pattern
through a constructive proof, showing that every such transposition
admits a topology-legal implementation of length at most $6d-4$. This
provides a reusable local primitive that improves by one CNOT over the
corresponding SWAP-based construction.

Building on this primitive, we introduce \emph{Global Transpositions}, a
deterministic synthesis method for arbitrary permutations. Each logical
transposition is assigned a topology-dependent CNOT cost, and the target
permutation is synthesized through a minimum-cost path in a weighted
Cayley graph of the symmetric group. The method is exact within the
restricted model of circuits obtained by concatenating independently
synthesized transpositions. Since the generator set includes all adjacent SWAPs, the CNOT count produced by Global Transpositions is never greater than that obtained by expanding an optimal SWAP routing into three-CNOT blocks.

Global Transpositions is more expressive than SWAP routing, but it still
requires the circuit to return to the permutation subgroup after every
logical transposition. AlphaCNOT reveals that this restriction can leave
additional optimization opportunities unexplored: for several target
permutations, direct synthesis produces circuits shorter than those
represented by the Global Transpositions model.

These additional reductions arise when the partial effects of multiple
logical operations are synthesized jointly. Intermediate prefix
transformations may temporarily leave the permutation subgroup
$S_n \subset GL(n,\mathbb{F}_2)$
and traverse more general linear reversible transformations before
returning to the target permutation. This allows partial effects to be
shared across the boundaries of independently synthesized
transpositions. We refer to this mechanism as \emph{Intermediate CNOT
Fusion}.

This phenomenon suggests a possible explanation for the separation
between SWAP-based and direct CNOT synthesis. SWAP routing and Global
Transpositions search within the permutation subgroup, whereas direct
CNOT synthesis can exploit the larger space of invertible binary linear
transformations. Intermediate states outside the permutation subgroup
can therefore provide shorter paths to a target permutation.

The contributions of this paper are threefold:
\begin{itemize}
    \item We give a constructive topology-legal implementation of a
    transposition at graph distance $d>1$ using at most $6d-4$ CNOT
    gates.

    \item We introduce \emph{Global Transpositions}, a deterministic
    synthesis method that computes a minimum-cost decomposition into
    independently synthesized logical transpositions and is never worse
    than an optimal SWAP routing expanded into CNOT gates.

    \item We experimentally investigate
    \emph{Intermediate CNOT Fusion}, through which direct CNOT synthesis
    combines partial effects of multiple logical operations using
    intermediate transformations outside the permutation subgroup.
\end{itemize}

Overall, our results show how AI-assisted optimization can support the
discovery of deterministic compilation principles.
In this context, AlphaCNOT acts not only as an optimizer, but as a discovery tool. It reveals
two levels of structure: a local remote-transposition saving, which we formalized
constructively and used to construct the Global Transpositions algorithm, and a more global fusion phenomenon, which is experimentally visible
but not yet fully characterized.

%% file: sections/background.tex
\section{Background}

This section reviews the background necessary for the remainder of the paper. We first recall the family of constant-depth quantum circuits for rotations and permutations introduced in~\cite{FaroPV24}, which serves as the baseline for our study. We then briefly describe the AI-based optimization framework employed throughout the paper~\cite{alphacnot}. In both cases, we focus only on the concepts and results directly relevant to our analysis, referring the reader to the original works for complete technical details.

\subsection{Constant-Depth Quantum Rotations and Permutations}

Let $R=(q_0,q_1,\ldots,q_{n-1})$ be a quantum register of $n$ qubits. A rotation of magnitude $k$ cyclically shifts the qubit states by $k$ positions, while a permutation rearranges the qubits according to an arbitrary permutation $\pi\in S_n$. These operations are fundamental building blocks in quantum algorithms and compilation, as they frequently arise when adapting logical circuits to the connectivity constraints of quantum hardware.

Faro \emph{et al.}~\cite{FaroPV24} introduced a constructive framework for implementing both rotations and arbitrary permutations by means of logspace-uniform families of quantum circuits having constant depth and linear size. The construction exploits the decomposition of a rotation into two reflections, each implemented as a layer of pairwise disjoint SWAP operations that can be executed completely in parallel. Consequently, every rotation is realized using only two layers of parallel SWAP gates, yielding a constant-depth circuit whose depth is independent of the size of the quantum register. 

The construction naturally extends from rotations to arbitrary permutations. Any permutation is first decomposed into disjoint cycles, and each cycle is implemented as a cyclic rotation acting only on the qubits belonging to that cycle. Since disjoint cycles operate on disjoint subsets of qubits, all cycle implementations can be executed simultaneously, preserving constant circuit depth while maintaining linear circuit size. This establishes that arbitrary qubit permutations belong to the complexity class $\mathrm{QNC}^{0}$. 

The circuits introduced in~\cite{FaroPV24} constitute the baseline considered throughout this paper. Their highly structured SWAP-based construction, based on layers of pairwise disjoint transpositions, provides a natural benchmark for investigating whether additional reductions can be achieved when synthesis is performed directly at the CNOT level. Rather than modifying the routing itself, our work analyzes how AI-based CNOT synthesis transforms these SWAP-optimal circuits and studies the structural properties of the resulting CNOT decompositions.

\subsection{AI-Based CNOT Optimization}

Artificial intelligence has recently emerged as a promising approach for
quantum circuit synthesis, optimization, and compilation. Reinforcement
learning methods have been applied to hardware-aware transpilation and CNOT
minimization \cite{kremer_practical_2024,romanello_cnot_2025,alphacnot}, while other works
have shown that AI-based search can discover non-trivial circuit structures
and compilation strategies. Examples include AlphaTensor for quantum circuit
optimization \cite{ruiz_quantum_2025}, reinforcement-learning-based
discovery of fault-tolerant state-preparation circuits
\cite{zen_quantum_2025}, reinforcement-learning based methods for synthesis and transpilation of permutation circuits \cite{AI_for_permutations_topologies}, RL-based methods for Clifford synthesis \cite{alphacliff, equivariantCliff}, and Q-PRESYN, where learned local circuit
transformations improve the outcome of a subsequent synthesis procedure
\cite{bosco_quantum_2026}. More recently, AI-driven search has also been
investigated for automated quantum algorithm discovery
\cite{finger_automated_2026}. These results suggest that AI-generated solutions
can be useful not only as optimized outputs, but also as objects from which
reusable compilation principles may be identified.

\paragraph{AlphaCNOT.}
AlphaCNOT \cite{alphacnot} is a model-based reinforcement learning algorithm that tackles the CNOT synthesis as a sequential planning problem. 
A CNOT circuit can be represented by its parity matrix \(A \in GL(n,\mathbb{F}_2)\). Consequently,  CNOT synthesis is addressed by searching for a sequence of CNOT operations that reduces \(A\) to the identity matrix. Each state of the search therefore corresponds to an invertible binary matrix, while each action corresponds to the application of a CNOT gate, implemented as an elementary row operation over \(\mathbb{F}_2\).

To learn efficient decompositions (i.e., equivalent circuits with lower gate count), AlphaCNOT combines deep neural networks with Monte Carlo Tree Search (MCTS) to explicitly explore multiple candidate synthesis trajectories. The MCTS procedure iterates through the standard phases of selection, expansion, simulation, and backpropagation, progressively focusing the search on promising branches of the search tree. Its exploration is guided by two neural networks: a policy network, which assigns prior probabilities to all possible CNOT actions, and a value network, which estimates the quality of the current matrix state, without performing an exhaustive exploration. This model-based formulation allows AlphaCNOT to exploit the known circuit dynamics and perform look-ahead search, rather than committing to a single sequence of heuristic-based actions.

AlphaCNOT tackles both the unconstrained version of the problem, also referred to as linear reversible synthesis, and the topology-aware one.
In the latter setting, the action space is restricted to CNOT gates supported by the coupling graph, defined \emph{a priori} in the problem's formulation.

Finally, since the search process is nondeterministic, AlphaCNOT can be used in a ``multi-shot'' fashion: given the same circuit, we can repeat the synthesis process multiple times and select the best outcome.

In what follows, we use AlphaCNOT primarily as an exploratory tool. The decompositions provided are used to explore and inspect  recurring structures, and to formulate hypotheses that can subsequently be
formalized as deterministic constructions. This methodology motivates the
local transposition construction and the intermediate CNOT-fusion phenomena
studied in the following sections.

\paragraph{Experimental Details.}

Since AlphaCNOT performs a stochastic search, repeated inference runs on the
same target may return different valid CNOT decompositions. In what follows, we report CNOT decompositions and counts either in a \emph{single-shot} setting, in which a single synthesis run is performed, or in a \emph{best-of-ten} setting, in which the model is independently sampled ten
times and the solution with the smallest CNOT count is retained. 
\begin{figure}[hb]
    \centering
    \begin{subfigure}[t]{0.45\linewidth}
        \centering
        \includegraphics[width=\linewidth]{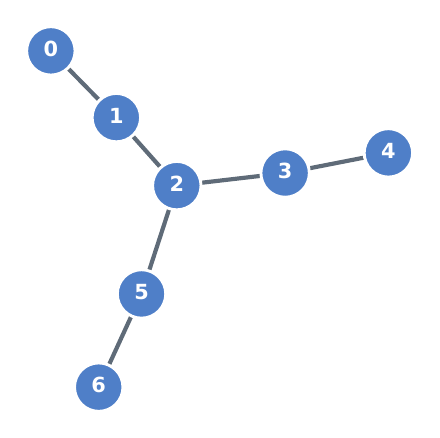}
        \caption{$7Y$ topology}
        \label{fig:topology-7y}
    \end{subfigure}
    \hfill
    \begin{subfigure}[t]{0.45\linewidth}
        \centering
        \includegraphics[width=\linewidth]{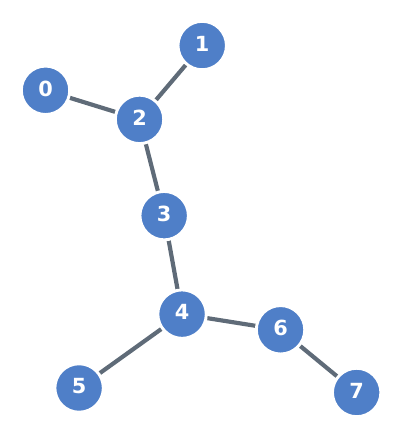}
        \caption{$8H$ topology}
        \label{fig:topology-8h}
    \end{subfigure}
    \caption{Examples of coupling topologies considered in our experiments.
    Vertices represent physical qubits, while edges denote the allowed
    two-qubit interactions. The full set of topologies can be found in \cite{kremer_practical_2024}.}
    \label{fig:representative-topologies}
\end{figure}

Figure~\ref{fig:representative-topologies} shows a few representative
coupling graphs used in the experiments.
We adopt the topology notation of~\cite{kremer_practical_2024}, in which
each coupling graph is identified by its number of qubits followed by a
letter denoting its connectivity family. Thus, for instance, $3L$ denotes a
three-qubit line, whereas $4T$ denotes a four-qubit T-shaped topology. We
extend the same naming convention to the other graph families considered in
our experiments.

The source code used to run AlphaCNOT, together with the trained models for
the coupling topologies considered in our experiments, is publicly available
in the AlphaCNOT GitHub repository.\footnote{\url{https://github.com/Jaccos01/AlphaCNOT}}

%% file: sections/discovery_local_primitives.tex
\section{Discovery of local primitives}
\label{sec:local_primitives}

As a starting point, we consider the problem of mapping a
target permutation consisting of a single transposition $\tau(i,j)$. Let
$d=d_G(i,j)$ denote the distance between the two endpoints in the coupling
graph $G$. A standard SWAP-based implementation moves one endpoint along a
shortest path, swaps it with the other, and moves it back. This requires
$2d-1$ adjacent SWAP gates and therefore $3(2d-1)=6d-3$ CNOT gates after
expanding each SWAP into three CNOTs. A closely related problem was studied
by Kutin et al.~\cite{computation_at_distance} in the linear
nearest-neighbor setting: their construction exchanges the two endpoints
of an $n$-wire line using $6n-9$ CNOT gates, which corresponds to $6d-3$
when the endpoint distance is $d=n-1$.

This bound, however, is not necessarily optimal with respect to CNOT count.
The smallest example is $\tau(0,2)$ on the three-qubit line $0-1-2$. The
SWAP-based construction requires three adjacent SWAPs, i.e., nine CNOT
gates, while the same permutation admits an eight-CNOT implementation, as
shown in Figure~\ref{fig:distance-two-transposition}.

\begin{figure}[t]
\centering
\includegraphics[width=\linewidth]{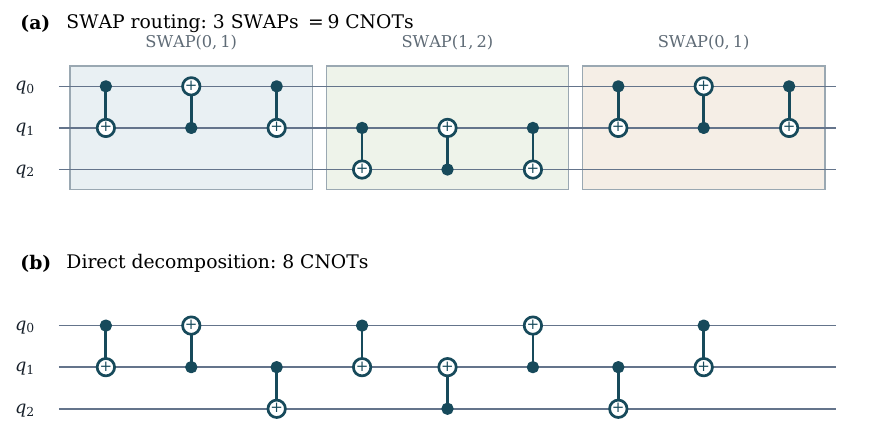}
\caption{The first local separation between SWAP routing and direct CNOT
synthesis.  The transposition $\tau(0,2)$ on the path $0-1-2$ costs $9$ CNOTs
when implemented as three adjacent SWAPs, but admits a direct topology-legal
implementation with $8$ CNOTs.}
\label{fig:distance-two-transposition}
\end{figure}

\paragraph{Empirical pattern.}
We evaluated the CNOT count found by AlphaCNOT on a variety of topologies for all transpositions. Interestingly, we noticed the same pattern (see Table \ref{tab:alphacnot-transposition-mean-by-topology-distance}): 
for most targets $\tau(i,j)$, AlphaCNOT finds a CNOT decomposition with CNOT count of
\[
    f(d) =
    \begin{cases}
        3, & d=1,\\
        6d-4, & d>1.
    \end{cases}
\]

\begin{table}[htb]
\centering
\setlength{\tabcolsep}{3pt}
\caption{Mean CNOT count found by AlphaCNOT (single shot) over all transpositions of distance $d$ for selected topologies as defined in \cite{kremer_practical_2024}. Values in bold represent averages that do not coincide with the formula $f(d)$, highlighting that in some cases AlphaCNOT was not able to obtain the observed upper bound.}
\label{tab:alphacnot-transposition-mean-by-topology-distance}
\begin{tabular}{lccccccc}
\hline
Topology & $d=1$ & $d=2$ & $d=3$ & $d=4$ & $d=5$ & $d=6$ & $d=7$ \\
\hline
Formula $f(d)$ & $3$ & $8$ & $14$ & $20$ & $26$ & $32$ & $38$ \\
4L & 3.00 & 8.00 & 14.00 &  &  &  &  \\
4Y & 3.00 & 8.00 &  &  &  &  &  \\
5L & 3.00 & 8.00 & 14.00 & 20.00 &  &  &  \\
5T & 3.00 & 8.00 & 14.00 &  &  &  &  \\
6L & 3.00 & 8.00 & 14.00 & 20.00 & 26.00 &  &  \\
6Y & 3.00 & 8.00 & 14.00 & 20.00 &  &  &  \\
7L & 3.00 & 8.00 & 14.00 & \textbf{20.33} & 26.00 & 32.00 &  \\
7Y & 3.00 & \textbf{8.17} & 14.00 & \textbf{20.67} &  &  &  \\
8F & 3.00 & 8.00 & 14.00 & \textbf{20.25} & 26.00 & 32.00 &  \\
8H & 3.00 & 8.00 & 14.00 & 20.00 & 26.00 &  &  \\
\hline
\end{tabular}%
\end{table}

The single-shot results in Table \ref{tab:alphacnot-transposition-mean-by-topology-distance} show that AlphaCNOT does not always find decompositions of length \(f(d)\) in a single inference run, as in a few cases (i.e. the ones in bold in Table \ref{tab:alphacnot-transposition-mean-by-topology-distance}), some average decompositions yield slightly increased CNOT counts. However, by increasing the computational budget by evaluating AlphaCNOT in a best-of-ten fashion, AlphaCNOT found exactly $f(d)$ for each transposition, suggesting that for each transposition $\tau(i,j)$ for $i,j$ of distance $d\geq2$, it is always possible to decompose $\tau(i,j)$ in $6d-4$ CNOTs.

Starting from this observation, we provide a constructive proof of the following theorem.

\begin{theorem}
Let $G$ be a topology and let $\tau(i,j)$ be a transposition whose
endpoints have distance $d_G(i,j)>1$ in $G$. Then $\tau(i,j)$ admits a topology-legal
CNOT implementation of length at most $6d-4$.
\end{theorem}

\begin{proof}
Let
\[
i=v_0,v_1,\ldots,v_d=j
\]
be a shortest path in $G$ from $i$ to $j$. Since every consecutive pair
$(v_k,v_{k+1})$ is connected by an edge of $G$, any CNOT circuit using only
edges of this path is automatically topology-legal on $G$. Thus it is enough
to construct the desired circuit on the line topology
\[
0-1-\cdots-d,
\]
after relabeling $v_k$ as $k$. In particular, it suffices to prove the
statement for the transposition $\tau(0,d)$ on this line.

Write $X_{a,b}$ for a CNOT with control $a$ and target $b$. All gate products
below are listed in application order. We proceed by induction on $d\ge2$. The base case is $d=2$. On the line
$0-1-2$, the following eight gates implement $\tau(0,2)$:
\[
X_{0,1}X_{1,0}X_{1,2}X_{0,1}
X_{2,1}X_{1,0}X_{1,2}X_{0,1}.
\]
Indeed, applying these gates to the row labels $(a,b,c)$ gives
\[
(a,b,c)
\mapsto (a,a+b,c)
\mapsto (b,a+b,c)
\mapsto (b,a+b,a+b+c)
\mapsto (b,a,a+b+c)
\]
\[
\mapsto (b,b+c,a+b+c)
\mapsto (c,b+c,a+b+c)
\mapsto (c,b+c,a)
\mapsto (c,b,a),
\]
which is exactly the transposition of the two endpoints.

For the induction step, assume that there is a CNOT circuit $C_{1,d}$
implementing $\tau(1,d)$ on the subpath $1-2-\cdots-d$, with length
$6(d-1)-4$. Let
\[
S_0 = X_{0,1}X_{1,0}X_{0,1}
\]
be the usual three-CNOT implementation of the adjacent transposition
$\tau(0,1)$. Then
\[
S_0\,C_{1,d}\,S_0
\]
implements $\tau(0,d)$, since at the permutation level
\[
(0\,1)(1\,d)(0\,1) = (0\,d).
\]
To see this directly, the two outer transpositions exchange the labels $0$
and $1$, so conjugating $(1,d)$ by $(0,1)$ simply replaces the endpoint
$1$ by $0$, while $d$ is fixed because $d>1$. Hence the resulting
transposition exchanges precisely $0$ and $d$.

Finally, the length of the circuit is
\[
3+\bigl(6(d-1)-4\bigr)+3=6d-4.
\]
This completes the induction and proves the claimed construction.
\end{proof}

A visualization of our construction for $d=4$ is provided in Figure \ref{fig:distance-four-transposition}, compared to an optimal output provided by AlphaCNOT. The linear behavior of the bound is also a natural consequence of the recursive construction. Increasing the distance by one amounts to surrounding the transposition at distance $d-1$ with two adjacent transpositions, each of which requires three CNOTs. Thus, each additional unit of distance contributes exactly six CNOTs, yielding the recurrence
$
L(d)=L(d-1)+6,
$
which, together with the base case $L(2)=8$, gives $L(d)=6d-4$.
 
It is important to emphasize that the construction above establishes only an upper bound, as we do not prove that the length $6d-4$ is optimal. Nevertheless, the available computational evidence  supports its optimality. In addition to the experimental results obtained with AlphaCNOT, we performed exhaustive brute-force searches on all graph topologies considered in this work with up to five qubits.
In all cases considered, no topology-legal implementation shorter than $6d-4$ was found. This evidence leads us to conjecture that the bound is tight, although a general proof of optimality remains open.

\begin{figure}[t]
\centering
\includegraphics[width=\linewidth]{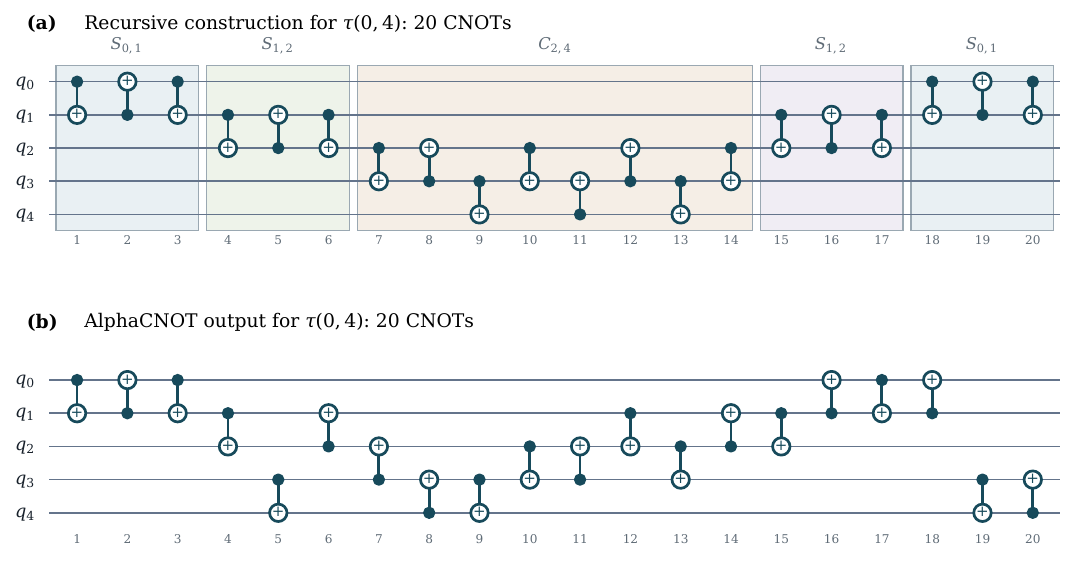}
\caption{Example of the recursive decomposition of $\tau(0,4)$ with distance $4$, compared to a decomposition found by AlphaCNOT. Note that in the decomposition below, it is easy to recognize the components corresponding to the two copies of the SWAP between $q_0$ and $q_1$ $S_{0,1}$ (i.e., the $6$ CNOTs between $q_0$ and $q_1$). However, the remaining transposition $\tau(1,4)$ cannot be easily decomposed  into $S_{1,2}$ and $C_{2,4}$.}
\label{fig:distance-four-transposition}
\end{figure}

At this point, the construction above assigns an explicit topology-dependent CNOT cost to every transposition of two qubits. This naturally suggests using remote transpositions as the basic building blocks for synthesizing arbitrary permutations, and choosing among them according to their certified CNOT costs. In the following section, we formalize this idea as a deterministic global synthesis algorithm.

%% file: sections/global_transposition_algorithm.tex
\section{Global Transpositions Algorithm}
\label{sec:global_transposition}

The previous result associates every logical transposition with a concrete topology-legal CNOT decomposition and, consequently, with an explicit CNOT cost. This allows us to move from the synthesis of individual transpositions to that of arbitrary permutations: instead of decomposing a permutation exclusively into adjacent SWAPs, we may search directly over all logical transpositions and weight each of them by the cost of its CNOT implementation.

Let $G=(V,E)$ be the coupling graph, with
$|V|=n$, and assume that every undirected edge supports both CNOT orientations.
For two vertices $i,j\in V$, let $d_G(i,j)$ be their graph distance.  The
previous construction gives a topology-legal CNOT implementation of the
transposition $\tau(i,j)$ with cost
\[
    w_G(i,j) =
    \begin{cases}
        3, & d_G(i,j)=1,\\
        6d_G(i,j)-4, & d_G(i,j)>1.
    \end{cases}
\]

Thus, every logical transposition can be assigned an explicit, certified CNOT
cost $w_G$, corresponding to a constructive decomposition in the same number of CNOTs.

Note that minimum-cost decompositions of permutations into transpositions with non-uniform costs have previously been studied in a general combinatorial setting \cite{farnoud_sorting_2012}. Our construction can be viewed as an instance of this framework in which transposition costs are induced by topology-constrained CNOT implementations.

\paragraph{Weighted Cayley graph.}
Define a weighted graph $\mathcal{T}_G$ whose vertices are the permutations in
$S_n$.  For every permutation $\pi\in S_n$ and every unordered pair
$\{i,j\}\subseteq V$, add an edge
\[
    \pi \longrightarrow \pi\circ\tau(i,j)
\]
with weight $w_G(i,j)$.  Since arbitrary transpositions generate $S_n$,
$\mathcal{T}_G$ is connected.  A path in $\mathcal{T}_G$ is a decomposition of
the target permutation into logical transpositions, and the path weight is a
guaranteed CNOT cost obtained by expanding each transposition with the
construction above.

We call the resulting method \emph{Global Transpositions}: it is global because
the decomposition is chosen over the whole target permutation, not independently
cycle by cycle.  In particular, if
\[
    \sigma = \tau_1\tau_2\cdots\tau_m
\]
is the path returned by the weighted search, then the CNOT cost reported by the
method is
\[
    \ell_{\mathrm{GT}}^G(\sigma)
    = \sum_{r=1}^{m} w_G(\tau_r).
\]

The corresponding concrete circuit is obtained by replacing each $\tau_r$ by
the CNOT block associated with a shortest path between its two endpoints.

\paragraph{Algorithm.}
For a fixed topology $G$, the algorithm is:
\begin{enumerate}
    \item compute all-pairs graph distances on $G$;
    \item assign each logical transposition $\tau(i,j)$ the weight $w_G(i,j)$;
    \item run Dijkstra from the identity permutation in the weighted Cayley
          graph $\mathcal{T}_G$;
    \item for a target $\sigma$, recover the predecessor chain from $\sigma$ to
          the identity, giving a minimum-weight transposition decomposition;
    \item expand each selected transposition into its topology-legal CNOT block.
\end{enumerate}

\paragraph{Guarantees.}
The output of Global Transpositions is exact for its own restricted model: among
all decompositions of $\sigma$ into logical transpositions with block costs
$w_G$, Dijkstra returns one with minimum total weight.  Moreover, this weight is
always an upper bound on the true topology-constrained CNOT optimum:
\[
    \ell_{\mathrm{CNOT}}^G(\sigma)
    \leq
    \ell_{\mathrm{GT}}^G(\sigma).
\]
Indeed, every weighted edge in $\mathcal{T}_G$ is backed by a concrete
topology-legal CNOT circuit.

The method also dominates optimal SWAP routing expanded gate by gate:
\[
    \ell_{\mathrm{GT}}^G(\sigma)
    \leq
    3\,\ell_{\mathrm{SWAP}}^G(\sigma).
\]
To see this, take an optimal adjacent-SWAP decomposition of $\sigma$.  Each
adjacent SWAP is precisely a transposition $\tau(i,j)$ with $d_G(i,j)=1$ and
weight $3$.  Since Global Transpositions minimizes over all transpositions, it
can reproduce that SWAP path and possibly improve on it by using longer
transpositions with the $6d-4$ identity.

These inequalities explain why this method is a useful middle baseline.  It is
strictly more expressive than SWAP routing, but it still returns to $S_n$ after
each transposition block.  Therefore it cannot represent CNOT-level fusions
between consecutive blocks.  Any residual gap
\[
    \ell_{\mathrm{GT}}^G(\sigma)
    -
    \ell_{\mathrm{CNOT}}^G(\sigma)
\]
or, empirically,
\[
    \ell_{\mathrm{GT}}^G(\sigma)
    -
    \ell_{\mathrm{AlphaCNOT}}^G(\sigma),
\]
is evidence for a phenomenon beyond independent transposition blocks.

\paragraph{Search-space and complexity comparison.}
The three relevant exact-search viewpoints differ mainly in the state space and
in the generator set.  Let $m=|E|$.  In Table \ref{tab:search-space-comparison}, we report the cost of the
most direct exhaustive algorithm for each model, ignoring constant-time details
of state encoding and hash-table lookup.

\begin{table}[t]
\centering
\caption{Naive exact search complexity for the three synthesis models.}
\label{tab:search-space-comparison}
\setlength{\tabcolsep}{3pt}
\resizebox{\linewidth}{!}{%
\begin{tabular}{lccccc}
\hline
Method & State space & Number of states & Generators & Naive algorithm & Worst-case time \\
\hline
Optimal SWAP
& $S_n$
& $n!$
& $|E|$ adjacent SWAPs
& BFS
& $O(n!\,|E|)$ \\
Global Transpositions
& $S_n$
& $n!$
& $\binom{n}{2}$ weighted transpositions
& Dijkstra
& $O\!\left(n!\,n^2\log(n!)\right)$ \\
Optimal CNOT
& $\mathrm{GL}(n,\mathbb{F}_2)$
& $\displaystyle\prod_{k=0}^{n-1}(2^n-2^k)$
& $2|E|$ oriented CNOTs
& BFS
& $O\!\left(|\mathrm{GL}(n,2)|\,|E|\right)$ \\
\hline
\end{tabular}
}
\end{table}

The all-pairs shortest-path distances required to define the transposition
weights can be computed in $O(n(n+|E|))$ time by running BFS from every
vertex. This preprocessing cost is negligible compared with the $n!$-state
search over the permutation space performed by Global Transpositions.

More generally, exact synthesis remains computationally challenging across
the different models considered here. Finding a minimum-length sequence of
adjacent SWAPs on an arbitrary connectivity graph corresponds to the Token
Swapping problem, which is NP-hard~\cite{miltzow_approximation_2016}.
Optimal CNOT synthesis is also computationally hard: hardness results are
known for topology-constrained CNOT synthesis
\cite{Optimal-space-depth-trade-off-cnots}, while CNOT-distance has recently
been shown to be NP-complete even under all-to-all connectivity
\cite{acuaviva_cnot-distance_2026}.

These complexity results motivate the development of more scalable
alternatives to exhaustive search. In this direction, the approximation
techniques developed for permutation decompositions with non-uniform
transposition costs~\cite{farnoud_sorting_2012} provide a natural starting
point for approximate variants of Global Transpositions that avoid explicit
exploration of the full permutation space $S_n$. Determining whether the
specific topology-induced cost function $w_G$ admits stronger approximation
guarantees or more efficient specialized algorithms remains an interesting
open question.

%% file: sections/intermediate_fusion.tex
\section{Beyond Transpositions: Intermediate CNOT Fusion}
\label{sec:beyond_transpositions}

Global Transpositions improves over SWAP routing by changing the elementary
objects used to synthesize a permutation.  Instead of restricting the search to
adjacent SWAPs, it searches over arbitrary logical transpositions and assigns a
topology-dependent CNOT cost to each of them.  The resulting permutation is
therefore implemented as a sequence of independently synthesized
transpositions.
The circuits produced by AlphaCNOT reveal that this construction does not
exhaust the available savings.  In several instances, direct CNOT synthesis
produces a circuit shorter than the complete Global Transpositions construction,
even though the latter has already selected its sequence of transpositions
globally. 

{Global Transpositions is optimal with respect to the decomposition model of transpositions. In particular, we have that given a connectivity map $G=(V,E)$  with $|V|=n$, and a permutation $\sigma\in S_n$, the cost found by GT is
\begin{equation*}
\ell^G_{\mathrm{GT}}(\sigma)
=
\min_{\substack{
    m\geq 0,\;\sigma=\tau_1\tau_2\cdots\tau_m\\
    \tau_r\text{ transposition},\;1\leq r\leq m
}}
\sum_{r=1}^{m} w_G(\tau_r),
\end{equation*}}

\begin{definition}[Intermediate CNOT fusion]
\label{def:intermediate-cnot-fusion}
Let $G$ be a connectivity map and let $\sigma \in S_n$.
We say that $\sigma$ admits an \emph{intermediate CNOT fusion}
(with respect to transpositions) on $G$ if there exists a
topology-legal CNOT circuit $\mathcal{C}$ implementing $\sigma$ such that
\begin{equation}
\label{eq:intermediate-cnot-fusion}
    \ell^G_{\mathcal{C}}(\sigma) < \ell^G_{\mathrm{GT}}(\sigma).
\end{equation}
Any such circuit $\mathcal{C}$ is called a \emph{fusion witness} for $\sigma$ on $G$.
Thus, intermediate CNOT fusion relative to GT occurs whenever
$\sigma$ admits a topology-legal CNOT implementation using fewer
CNOT gates than the minimum-cost concatenation of transposition
blocks considered by the GT model.
\end{definition}

Note that this definition is specifically relative to transpositions, whose elementary
permutation blocks are cycles of length two, i.e., logical
transpositions, with their constructive implementation costs $w_G$ (assuming they are optimal).
The same comparison could possibly be  extended to synthesis models whose block
libraries also include longer cycles (for example, the set of transpositions and 3-cycles, etc.), allowing for a more general definition of \textit{intermediate fusion}. However, while for transpositions we were able to obtain a constructive bound of the cost $w_G$, for longer cycles we would need to develop different tools, or to compute optimal costs of each cycle. For this reason, in what follows we will consider only fusions with respect to transpositions.

\paragraph{Minimal example.}

Consider the line topology $4L$ and the target permutation
\[
    \sigma=[2,3,1,0].
\]
Global Transpositions returns
\[
    \tau(1,2)\,\tau(0,1)\,\tau(1,3),
\]
whose three blocks cost $3+3+8=14$ CNOTs.  The first two blocks are adjacent
transpositions, while the last one is a distance-two transposition implemented
with the construction of Section~\ref{sec:global_transposition}.  AlphaCNOT, when applied directly to the full permutation, returns the
$13$-CNOT circuit shown in Figure~\ref{fig:intermediate-fusion-example}.
For this instance, exhaustive bidirectional BFS over topology-legal CNOT
circuits certifies that $13$ CNOTs are optimal.

\begin{figure}[t]
\centering
\includegraphics[width=\linewidth]{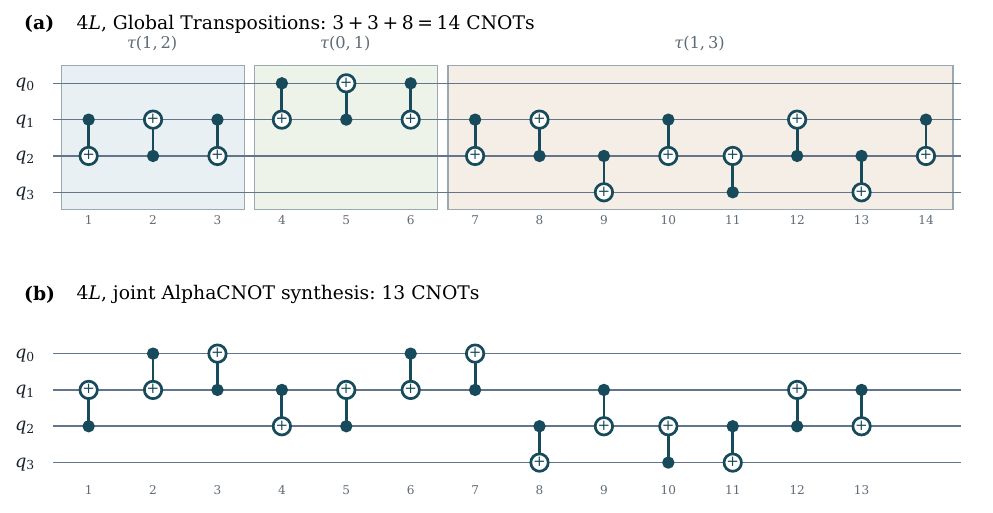}
\caption{Intermediate CNOT fusion on $4L$ for
$\sigma=[2,3,1,0]$.  Global Transpositions realizes three independently
synthesized transpositions using $14$ CNOTs, whereas direct AlphaCNOT synthesis
realizes an optimal decomposition in $13$ gates.}
\label{fig:intermediate-fusion-example}
\end{figure}

This example separates intermediate fusion from the remote-transposition
construction studied in the previous section.  The latter reduces the cost of
one logical transposition.  Here, all three transpositions have already been
assigned the individual costs used by the GT construction; the additional CNOT
disappears only when the complete permutation is synthesized as one circuit.
The improvement is numerically small, but it exposes a limitation of any method
whose total cost is obtained by summing independently synthesized
transposition blocks.

\begin{table}[t]
\centering
\caption{Mean CNOT count over $100$ target permutations per constrained
topology (for $4$ qubits, all $24$ permutations are considered). AlphaCNOT results are obtained with $10$ repetitions. Results for Optimal CNOT are obtained with bidirectional BFS.}
\label{tab:fusion-empirical-by-topology}
\begin{tabular}{crrrr}
\hline
topology & SWAP$\times3$ & Global Transpositions & AlphaCNOT & Optimal CNOT \\
\hline
$4L$  & $8.64$  & $8.18$  & $8.03$  & $8.03$  \\
$4Y$  & $8.04$  & $7.83$  & $7.59$  & $7.59$  \\
$5L$  & $15.03$ & $14.21$ & $13.73$ & $13.71$ \\
$5T$  & $13.50$ & $12.92$ & $12.46$ & $12.42$ \\
$6L$  & $22.59$ & $21.24$ & $20.37$ & -- \\
$6T$  & $19.20$ & $18.28$ & $17.67$ & -- \\
$6Y$  & $18.78$ & $17.86$ & $17.26$ & -- \\
$7L$  & $32.28$ & $30.23$ & $28.97$ & -- \\
$7T$  & $27.45$ & $25.89$ & $25.01$ & -- \\
$7Y$  & $26.40$ & $24.95$ & $24.53$ & -- \\
$8F$  & $37.71$ & $35.34$ & $34.07$ & -- \\
$8H$  & $32.40$ & $30.82$ & $29.92$ & -- \\
$8T$  & $38.67$ & $36.36$ & $35.23$ & -- \\
$8T2$ & $33.27$ & $31.57$ & $30.71$ & -- \\

\hline
\end{tabular}
\end{table}

\paragraph{Full evaluation.}
We now evaluate the average number of CNOTs found by optimal SWAP algorithm, Global Transpositions, and AlphaCNOT over different topologies. For each topology, we randomly select $100$ permutations of the corresponding size (for $4L$ and $4Y$ we consider all the $24$ permutations instead). Results for AlphaCNOT are obtained with a best-of-ten strategy, i.e. we execute the model $10$ times and  select the best solution. In addition, we report the average optimal CNOT count obtained by bidirectional BFS for up to $5$ qubits. We report our results in Table \ref{tab:fusion-empirical-by-topology}.

AlphaCNOT achieves a lower mean CNOT count than Global Transpositions on every
topology in the table.  Across the full set sampled targets, the
mean decreases from $23.85$ CNOTs for SWAP synthesis to $22.55$ for Global
Transpositions, and then to $21.83$ for AlphaCNOT.  Thus, the transposition model
explains an average reduction of $1.31$ CNOTs relative to SWAP synthesis, while
direct synthesis finds a further reduction of $0.72$ CNOTs, or $3.2\%$ relative
to Global Transpositions.  The latter gap is present on every tested constrained
topology and ranges from $0.15$ to $1.27$ CNOTs per target on average.

These results support a two-stage interpretation of the structures exposed by
AlphaCNOT.  Remote-transposition synthesis explains why a logical
transposition need not be implemented as a chain of full SWAPs.  Intermediate
fusion explains the remaining empirical gap: even after the transpositions
have been selected globally and assigned improved individual costs, treating
them as independent blocks leaves additional CNOT reductions unexplored.

\paragraph{Additional examples.}

Finally, we show in Figure \ref{fig:esempi_visualizzati} some examples and visualizations of different permutations, with the solutions found by Global Transpositions and AlphaCNOT. Note that while AlphaCNOT is not optimal, it can find non-trivial CNOT reductions that cannot be obtained by performing optimal synthesis over $S_n$. This suggests that Intermediate CNOT fusions play a major role in obtaining efficient CNOT decompositions of permutation circuits.

\begin{figure}[htbp]
    \centering

    \begin{subfigure}[t]{0.49\textwidth}
        \centering
        \includegraphics[width=\linewidth]{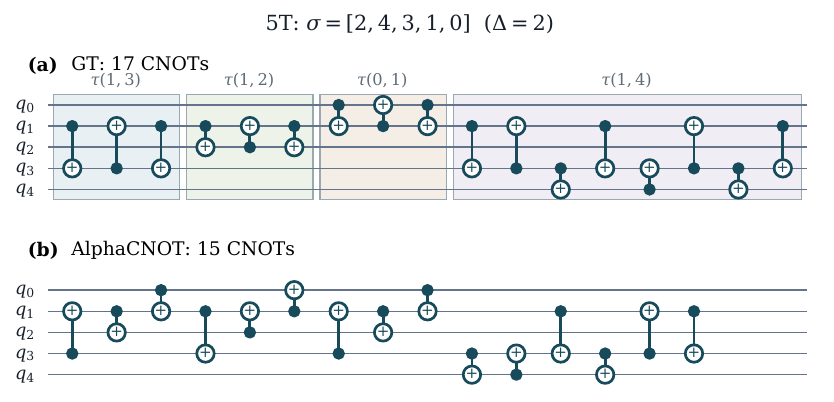}
        \caption{Decomposition of  $[2,4,3,1,0]$ over $5T$.}
        \label{fig:esempio2}
    \end{subfigure}   
    \hfill
 \begin{subfigure}[t]{0.49\textwidth}
        \centering
        \includegraphics[width=\linewidth]{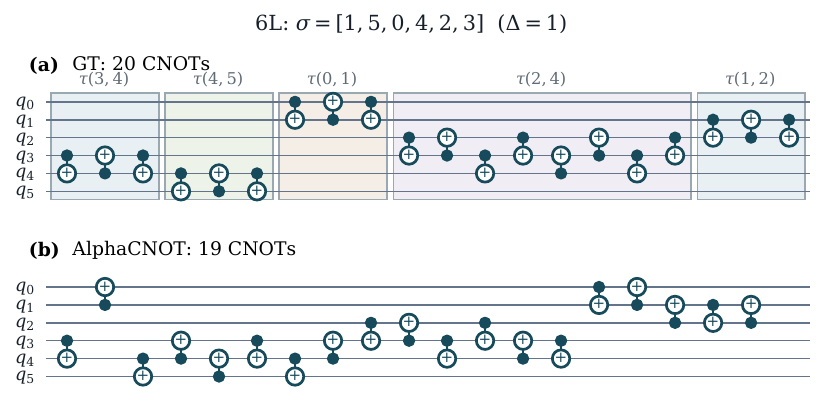}
        \caption{Decomposition of  $[1,5,0,4,2,3]$ over $6L$.}
        \label{fig:esempio1}
    \end{subfigure}

    \vspace{0.7cm}

    \begin{subfigure}[t]{0.49\textwidth}
        \centering
        \includegraphics[width=\linewidth]{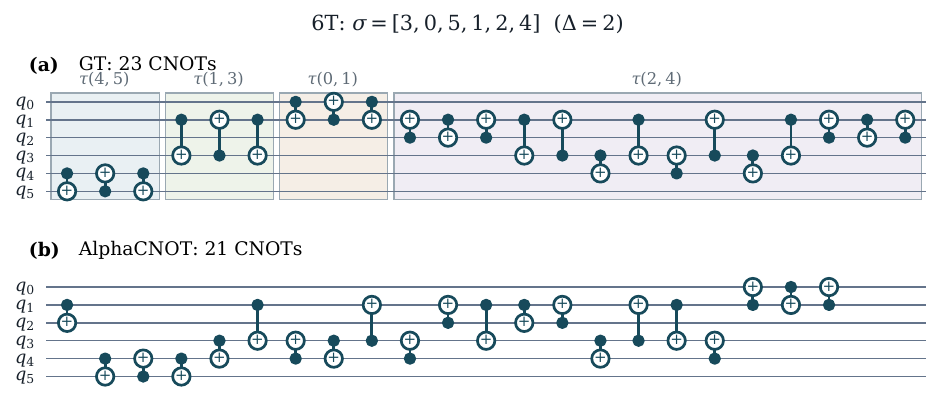}
        \caption{Decomposition of  $[3,0,5,1,2,4]$ over $6T$.}
        \label{fig:esempio3}
    \end{subfigure}
    \hfill
    \begin{subfigure}[t]{0.49\textwidth}
        \centering
        \includegraphics[width=\linewidth]{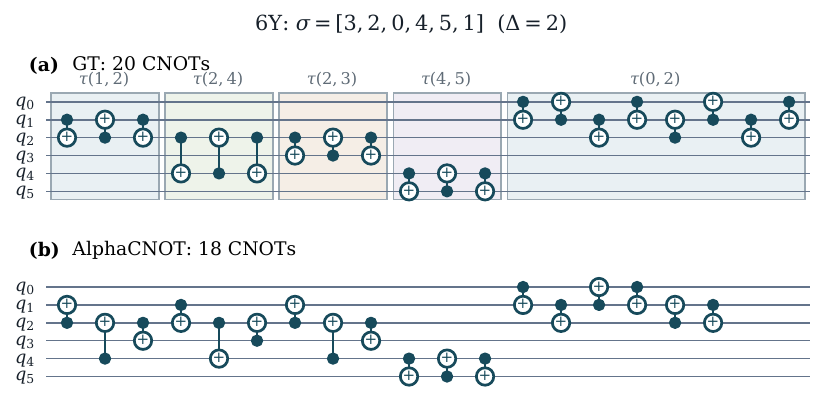}
        \caption{Decomposition of  $[3,2,0,4,5,1]$ over $6Y$.}
        \label{fig:esempio6t}
    \vspace{0.7cm}
    \end{subfigure}
        \begin{subfigure}[t]{0.98\textwidth}
        \centering
        \includegraphics[width=\linewidth]{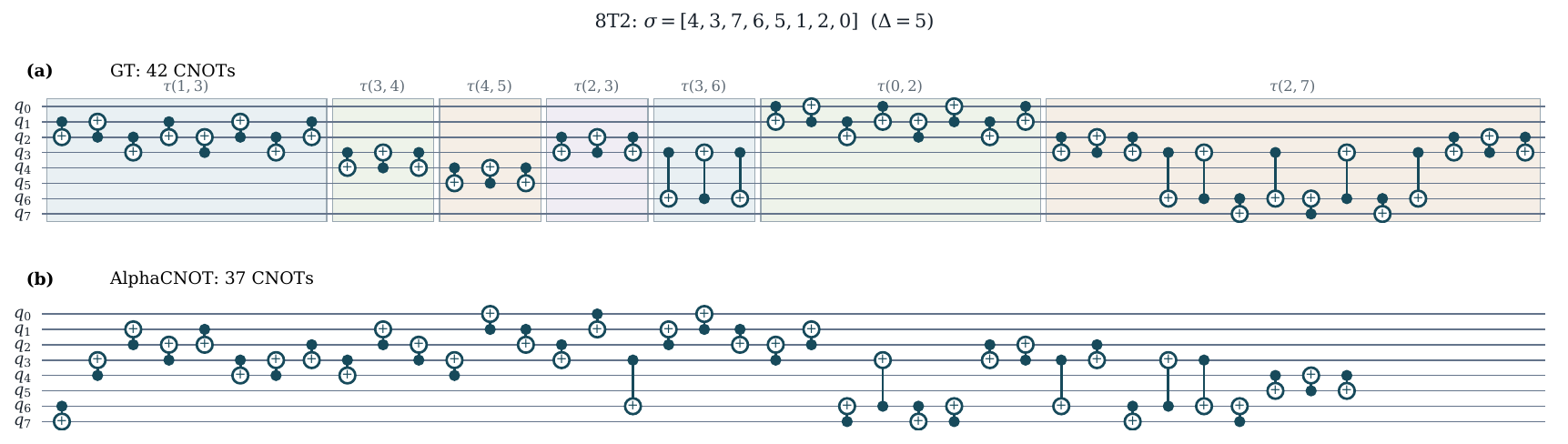}
        \caption{Decomposition of  $[4, 3, 7, 6, 5, 1, 2, 0]$ over $8T2$.}
        \label{fig:esempio8t2}
    \end{subfigure}

    \caption{Visualizations of some decompositions found by Global Transpositions (GT) and AlphaCNOT with non-trivial intermediate CNOT fusions.}
    \label{fig:esempi_visualizzati}
\end{figure}

%% file: sections/discussion.tex
\section{Discussion and Future Directions}
In this work, we used AlphaCNOT to investigate CNOT synthesis for
depth-optimal quantum rotation and permutation circuits. Our analysis revealed two distinct
structures.
First, the circuits produced by AlphaCNOT exposed a recurring
pattern for remote transpositions, which we formalized as a constructive
implementation using at most $6d-4$ CNOT gates for endpoints at graph
distance $d>1$. This result led to \emph{Global Transpositions}, a
deterministic synthesis method that optimizes the decomposition of an
entire permutation into topology-aware logical transpositions.
Beyond this, AlphaCNOT reveals additional reductions that cannot be captured by independently optimized transposition blocks. By exploiting intermediate linear transformations outside the permutation subgroup, it can combine effects across transposition boundaries, a mechanism we call \emph{Intermediate CNOT Fusion}. This shows that the gap between SWAP-optimal and CNOT-efficient synthesis stems not only from cheaper transpositions, but also from the richer intermediate state space available at the CNOT level.

More broadly, these results illustrate a role for AlphaCNOT beyond
black-box circuit optimization. Its generated circuits provide concrete
objects that can be inspected to identify recurring patterns and
formulate structural hypotheses. In the case of remote transpositions,
such a pattern led to a constructive proof and a deterministic synthesis
algorithm; in the case of Intermediate CNOT Fusion, it exposed a
limitation of approaches that restrict their intermediate states to the
permutation subgroup.

Several questions remain open. The most immediate is a theoretical
characterization of Intermediate CNOT Fusion: determining when fusion is
possible, how large the resulting CNOT savings can be, and whether such
opportunities can be identified directly from the target permutation
and the coupling graph. A related question is whether intermediate
non-permutation transformations are necessary for CNOT-optimal
realizations of some permutation circuits, or whether equally short
circuits confined to more restricted classes of intermediate
transformations always exist. More generally, these observations suggest
the possibility of deriving new deterministic synthesis rules from the
patterns exposed by AlphaCNOT and extending the same methodology to
other circuit families and quantum compilation problems. 